%% file: sample-sigconf.tex
\documentclass[sigconf]{acmart}

\AtBeginDocument{%
  }

\usepackage{amsmath}
\usepackage{comment}
\usepackage{algorithmic}
\usepackage{enumitem}
\usepackage{amsmath}
\usepackage{listings}
\usepackage{xcolor}
\usepackage{bbm} 
\usepackage{multirow}
\usepackage{algorithm}
\include{math_command}

\newtheorem{definition}{Definition}

\lstdefinelanguage{pseudo}{
  keywords={def, for, in, if, return, else, zero_grad, step, Args},
  keywordstyle=\bfseries\color{blue},
  ndkeywords={logits, labels, edge_index, calibration_model, gin_detector, calibration_optimizer, gin_optimizer, temperature, calibrated_logits, group_weights, ece_loss, ce_loss, total_loss, deg_info},
  ndkeywordstyle=\color{teal},
  identifierstyle=\color{black},
  comment=[l]{\#},
  commentstyle=\color{gray}\itshape,
  stringstyle=\color{red},
  sensitive=true,
  morecomment=[s][\color{purple}]{"""}{"""}
}

\begin{document}

%%
%% The "title" command has an optional parameter,
%% allowing the author to define a "short title" to be used in page headers.
\title{Enhance GNNs with Reliable Confidence Estimation via Adversarial Calibration Learning}

%%
%% The "author" command and its associated commands are used to define
%% the authors and their affiliations.
%% Of note is the shared affiliation of the first two authors, and the
%% "authornote" and "authornotemark" commands
%% used to denote shared contribution to the research.
\author{Yilong Wang, Jiahao Zhang, Tianxiang Zhao, Suhang Wang}
\affiliation{%
  \institution{The Pennsylvania State University}
  \city{University Park}
  \country{USA}
  }
\email{{yvw576, jiahao.zhang, tkz5084, szw494}@psu.edu}
% \author{Jiahao Zhang}
% \affiliation{%
%   \institution{The Pennsylvania State University}
%   \city{State College}
%   \country{USA}}
% \email{jiahao.zhang@psu.edu}
% \author{Tianxiang Zhao}
% \affiliation{%
%   \institution{The Pennsylvania State University}
%   \city{State College}
%   \country{USA}}
% \email{tkz5084@psu.edu}
% \author{Suhang Wang}
% \affiliation{%
%   \institution{The Pennsylvania State University}
%   \city{State College}
%   \country{USA}}
% \email{szw494@psu.edu}

%%
%% By default, the full list of authors will be used in the page
%% headers. Often, this list is too long, and will overlap
%% other information printed in the page headers. This command allows
%% the author to define a more concise list
%% of authors' names for this purpose.
% \renewcommand{\shortauthors}{Trovato et al.}

%%
%% The abstract is a short summary of the work to be presented in the
%% article.
\begin{abstract}
Despite their impressive predictive performance, GNNs often exhibit poor confidence calibration, i.e., their predicted confidence scores do not accurately reflect true correctness likelihood. This issue raises concerns about their reliability in high-stakes domains such as fraud detection, and risk assessment, where well-calibrated predictions are essential for decision-making. To ensure trustworthy predictions, several GNN calibration methods are proposed. Though they can improve global calibration, our experiments reveal that they often fail to generalize across different node groups, leading to inaccurate confidence in node groups with different degree levels, classes, and local structures. In certain cases, they even degrade calibration compared to the original uncalibrated GNN. To address this challenge, we propose a novel AdvCali framework that adaptively enhances calibration across different node groups. Our method leverages adversarial training to automatically identify mis-calibrated node groups and applies a differentiable Group Expected Calibration Error (ECE) loss term to refine confidence estimation within these groups. This allows the model to dynamically adjust its calibration strategy without relying on dataset-specific prior knowledge about miscalibrated subgroups. Extensive experiments on real-world datasets demonstrate that our approach not only improves global calibration but also significantly enhances calibration within groups defined by feature similarity, topology, and connectivity, outperforming previous methods and demonstrating its effectiveness in practical scenarios.
\end{abstract}

\maketitle

\section{Introduction}

Graphs are a fundamental data structure that naturally represents relationships in real-world applications, such as social networks~\cite{fan2019graph,zhang2024graph}, purchase networks~\cite{jin2024amazon, zhang2024linear}, citation networks~\cite{sen2008collective,hu2020open}, and biological systems~\cite{strokach2020fast,muzio2021biological}. Node classification is one of the most essential graph mining tasks, serving as the foundation for numerous downstream applications. Graph Neural Networks (GNNs)~\cite{kipf2017semisupervised,velickovic2018GAT,gilmer2017neural, wang2025bridging} have emerged as powerful tools for node classification by leveraging message-passing mechanisms to aggregate information from neighboring nodes. This approach has led to significant improvements in prediction accuracy. However, despite their strong performance, GNNs often suffer from poorly calibrated predictions~\cite{wang2021confident,hsu2022makes}, meaning that the predicted confidence scores do not accurately reflect the true correctness likelihood. This mismatch raises concerns about the trustworthiness of GNN predictions, particularly in high-stakes applications such as fraud detection~\cite{dou2020enhancing,zhang2022efraudcom}, medical diagnosis~\cite{sun2020disease,hao2021medto}, and risk assessment~\cite{cheng2019risk,bi2022company}, where reliable decision-making is critical.

Hence, confidence calibration on GNNs~\cite{wang2021confident} is attracting increasing attention and many efforts have been taken. Generally, existing GNN calibration methods can be categorized into two categories, i.e., in-training~\cite{wang2022gcl,stadler2021graph} and post-hoc methods ~\cite{wang2021confident,hsu2022makes,yang2024calibrating}. In-training methods jointly train with the node classifier and optimize calibration results by adding regularization terms; while post-hoc methods freeze the well-trained GNN classifier and train an independent temperature predictor to learn the node-wise temperature to scale the logits learned by the GNN classifier to achieve better calibration results. While these methods have shown promising results in improving global calibration, our experiments in Section ~\ref{PreliminaryAnalysis} reveal that they often have unfair performance across different node groups. In particular, calibration performance degrades on high-degree nodes, certain node classes, and structurally distinct subgraphs, sometimes even increasing uncertainty compared to uncalibrated GNNs. This failure of existing calibration methods on specific node subgroups can be attributed to several key factors. First, most calibration approaches focus solely on optimizing global calibration metrics across the entire graph, neglecting calibration performance at the subgroup level. Since these methods do not explicitly account for differences in calibration across subclasses, certain node groups—such as high-degree nodes, specific classes, or distinct subgraphs—may suffer from systematic miscalibration despite improvements in overall confidence alignment. Second, most post-hoc GNN calibration models rely on cross-entropy loss as the primary objective for calibration adjustment. However, cross-entropy loss primarily encourages correct classification rather than directly improving confidence alignment, making it difficult to achieve well-calibrated probability estimates. 

Global calibrated models with group-level miscalibration provide a fake sense of trustworthiness and can lead to critical failures in real-world applications where certain node groups require higher reliability. For example, in medical diagnosis systems~\cite{sun2020disease,lu2021weighted}, modeled as patient-disease interaction graphs, nodes representing rare diseases (minority classes) require precise uncertainty estimates to prevent misdiagnoses. In financial fraud detection, high-degree nodes, such as accounts with numerous transactions, often represent high-risk entities~\cite{zhang2021fraudre,han2025mitigating}, yet existing calibration methods may amplify uncertainty for these nodes, increasing the likelihood of undetected fraudulent activities. These examples demonstrate that calibration bias across node groups is not just a statistical issue but a fundamental problem of fairness and safety. When models systematically miscalibrate predictions for specific node groups—such as minority classes, hub nodes, or dense subgraphs—they risk reinforcing biases in decision-making or triggering cascading failures in critical systems. Thus, ensuring subgroup-aware calibration is not only a technical challenge but also essential for deploying trustworthy and reliable GNNs in real-world applications. 

Therefore, in this work, we study a novel problem of mitigating calibration bias across different node groups and achieving more reliable confidence estimation in GNNs. In essence, we are faced with two challenges: \textbf{(i)} given a dataset and a GNN classifier on the dataset, we lack the knowledge of which groups the calibration will fail. How to identify groups that will fail to improve the calibration remains a question; and \textbf{(ii)} with the group identified, how to explicitly optimize the group-wise calibration? To address these two challenges, we propose a novel adversarial calibration framework that dynamically identifies and prioritizes miscalibrated subgroups during training. Our approach introduces a group detector that adversarially learns to identify miscalibrated node groups by maximizing the calibration loss within each detected group. Our group detector autonomously identifies groups with calibration failures, rather than relying on predefined rules or human assumptions (e.g., "high-degree nodes are problematic"), which ensures adaptability across diverse datasets and graph structures. Once the miscalibrated groups are identified, we introduce a differentiable Group Expected Calibration Error (ECE) loss to guide the calibration learning process. Unlike traditional cross-entropy loss, which primarily encourages classification accuracy while often pushing logits toward extreme values, our Group ECE loss provides a more direct estimation of confidence-to-accuracy misalignment. This enables a more fine-grained calibration adjustment, improving reliability across node subgroups. Our main contributions are:
\begin{itemize}[leftmargin=*]
    \item We are the first to identify and systematically analyze the calibration bias across different node groups on graphs, highlighting its impact on GNN reliability beyond global calibration.
    \item We propose the first adversarial group calibration framework for GNNs. Our method automatically detects miscalibrated subgroups and optimizes a differentiable Group-ECE loss to minimize confidence-accuracy gaps within these groups, achieving balanced calibration across diverse subgroups.
    \item  We conduct experiments on eight widely used benchmarks. The result demonstrates that our method effectively improves GNN calibration, outperforming state-of-the-art approaches in both global and subgroup calibration metrics.
\end{itemize}

\section{Related works}

\noindent\textbf{Graph Neural Networks.} GNNs leverage graph structures to perform convolution operations for effective feature extraction. They are primarily divided into spectral-based~\cite{defferrard2016convolutional,xu2018graph,he2021bernnet,wang2022powerful} and spatial-based approaches~\cite{niepert2016learning,atwood2016diffusion,hamilton2017inductive,gilmer2017neural}. Spectral GNNs operate in the spectral domain using graph Laplacian filters, while spatial GNNs aggregate information from neighboring nodes. Notable advances include GAT~\cite{velickovic2018GAT}, which adopts attention mechanisms to aggregate neighbors, and GIN~\cite{xu2018powerful}, which enhances expressive power by leveraging injective aggregation functions. However, despite their success, GNNs often face challenges in ensuring not only high accuracy but also well-calibrated confidence estimates, which are crucial for reliable decision-making in real-world applications. This challenge has led to increasing research on confidence calibration for neural networks, aiming to ensure that predicted confidence scores accurately reflect correctness probabilities.

\noindent\textbf{General Calibration Methods.} Calibration Methods for Neural Networks aim to ensure that a model’s predicted confidence scores accurately reflect the true correctness likelihood, enhancing its trustworthiness in decision-making. Existing calibration methods for neural networks can be broadly categorized into in-processing and post-hoc approaches. In-processing methods incorporate calibration, directly into the training process, often by incorporating regularization terms or modifying the loss function.  Notable techniques include Focal loss~\cite{mukhoti2020calibrating,wang2021rethinking,ghosh2022adafocal,tao2023dual} and Maximum Mean Calibration Error (MMCE)~\cite{kumar2018trainable,widmann2019calibration}. Post-hoc methods, on the other hand, recalibrate a pre-trained model without modifying its internal parameters. These methods learn the temperature to adjust the model’s output confidence scores to better align with actual correctness probabilities. Widely used approaches include Temperature Scaling ~\cite{guo2017calibration}, and Vector Scaling ~\cite{guo2017calibration}.

\noindent\textbf{Calibration Methods for GNNs.} In recent years, there has been a growing interest in developing post-hoc calibration methods specifically for GNNs. CaGCN~\cite{wang2021confident} was the first work to address calibration in graph-based learning by introducing a separate GNN to learn node-wise temperature parameters. To mitigate GNNs' common tendency toward underconfidence, it promotes sparser output distributions. GATS~\cite{hsu2022makes} extends this approach by explicitly encoding multiple graph-specific factors into temperature learning, including neighborhood confidence variation, distance to the training set, and global bias. SimCali~\cite{tang2024simcalib} further incorporates node-wise similarity scores, leveraging graph structural properties to refine confidence estimates. DCGC~\cite{yang2024calibrating} adopts a data-centric approach, increasing the influence of homophilous connections and structurally significant edges during calibration, thereby improving confidence alignment. While these methods improve global-level calibration on graphs, they do not explicitly address subgroup-specific miscalibration, where confidence alignment varies across different node groups (e.g., high-degree nodes or underrepresented classes). This can lead to biased uncertainty estimates, negatively impacting fairness and robustness in real-world applications.

\section{Preliminary}

\subsection{Notation and Background} \label{sec:background}
\noindent\textbf{Notations.} In this paper, we focus on confidence calibration~\cite{naeini2015obtaining} in the node classification~\cite{wu2019simplifying,chen2020simple} task on graphs. Let \( \mathcal{G} = (\mathcal{V}, \mathbf{A}, \mathbf{X}, \mathcal{Y}) \) be an attributed graph, where \( \mathcal{V} = \{ v_1, \dots, v_N \} \) is the set of \( N \) nodes. The adjacency matrix is given by \( \mathbf{A} \in \mathbb{R}^{N \times N} \), where \( A_{i,j} = 1 \) if node \( v_i \) and node \( v_j \) are connected, and \( A_{i,j} = 0 \) otherwise. The node feature matrix is denoted as \( \mathbf{X} \in \mathbb{R}^{N \times d} \), where \( \mathbf{X}_i \in \mathbb{R}^{d} \) represents the attribute vector of node \( v_i \). The label set is defined as \( \mathcal{Y} = \{ y_1, \dots, y_N \} \), where \( y_i \in \{ 1, \dots, C \} \) represents the ground truth label of node \( v_i \) for a classification task with \( C \) classes. In the semi-supervised setting, only a subset of nodes in the graph are labeled. We denote the labeled node set as \( \mathcal{V}_L \subset \mathcal{V} \). In node classification task, we aim to learn a function \( f_\theta: \mathbb{R}^{N \times d} \to \mathbb{R}^{N \times C} \) on labeled set \( \mathcal{V}_L \) that maps node features and graph structure to class probabilities \( {\mathbf{P}} \in \mathbb{R}^{N \times C} \).

\noindent\textbf{Calibration on GNN.} A well-calibrated GNN ensures that the predicted confidence score accurately reflects the true correctness likelihood of the prediction ~\cite{wang2021confident}. For example, if a GNN assigns an average confidence score of 0.6 to 100 predictions, then approximately 60 of them should be correct for a well-calibrated model. Formally, a perfectly calibrated classifier satisfies the following condition:
\begin{equation}
\forall c \in [0,1], \quad \mathbb{P}(y_i = \hat{y}_i | \hat{c}_i = c) = c,
\end{equation}
where \( \hat{y}_i = \arg\max_{c}{P}_{i,c} \) is the predicted class label for node \( v_i \), and \( \hat{c}_i = \max_{c}{P}_{i,c} \) is the predicted confidence score, i.e., the maximum softmax probability across all classes. A model is perfectly calibrated when, for all confidence values \( c \in [0, 1] \), the empirical accuracy of predictions with confidence \( c \) is also \( c \). Expected calibration error (ECE)~\cite{naeini2015obtaining} is widely used to evaluate calibration performance. It is formally defined as:
\begin{definition}[Expected calibration error (ECE)]\label{def:orig_ece} Given a set of predictions made by a classifier $f$ and the corresponding ground truth, ECE first partitions predictions into \( M \) equally spaced confidence intervals. Each bin \( \mathcal{B}_m \) contains all samples whose predicted confidence falls within the interval:  $$\mathcal{B}_m := \left\{ v_i \in \mathcal{V} \mid \frac{m-1}{M} < \hat{c}_i \leq \frac{m}{M} \right\}.$$
The ECE score for $f$ is then defined as
\begin{equation}
\mathcal{L}_{\mathrm{ECE}} := \sum_{m=1}^{M} \frac{|\mathcal{B}_m|}{|\mathcal{V}|} \Big| \mathrm{Acc}(\mathcal{B}_m) - \mathrm{Conf}(\mathcal{B}_m) \Big|,
\end{equation}
where $|\mathcal{V}|$ is the number of nodes in $\mathcal{V}$, $\mathrm{Acc}(\mathcal{B}_m)$ is the average prediction accuracy for samples inside $\mathcal{B}_m$, and $\mathrm{Conf}(\mathcal{B}_m)$ is the average confidence for samples inside $\mathcal{B}_m$, which are given as
\begin{equation}
\mathrm{Acc}(\mathcal{B}_m) := \frac{1}{|\mathcal{B}_m|} \sum_{v_i \in \mathcal{B}_m} \mathbbm{1} (y_i = \hat{y}_i),~~ \mathrm{Conf}(\mathcal{B}_m) := \frac{1}{|\mathcal{B}_m|} \sum_{v_i \in \mathcal{B}_m} \hat{c}_i.
\end{equation}
\end{definition}

\subsection{Preliminary Analysis}\label{PreliminaryAnalysis}

In this subsection, we conduct preliminary experiments on real-world datasets to explore the factors influencing confidence calibration and examine whether different datasets exhibit similar sensitivity to these factors. Specifically, we train a GCN classifier on the Cora~\cite{yang2016revisiting} and Pubmed~\cite{yang2016revisiting} datasets using their predefined train/validation splits. We then apply GATS~\cite{hsu2022makes}, a widely used GNN calibration model, to calibrate the confidence scores of the trained GCN classifier. To evaluate calibration performance at both the global and subgroup levels, we analyze four key metrics, as illustrated in Figure~\ref{fig:pre1}: (i) ECE – the Expected Calibration Error computed over all nodes in the graph. (ii) Degree ECE – the calibration error specifically measured on the top 25\% highest-degree nodes. (iii) Class ECE – the calibration error computed for nodes belonging to class label 0. (iv) Subgraph ECE – the calibration error within the largest subgraph obtained using the Louvain algorithm~\cite{blondel2008fast}. These metrics allow us to examine whether calibration methods consistently improve confidence alignment across different structural and semantic groups or whether certain subgroups remain miscalibrated even after post-hoc temperature scaling. The results are shown in Fig.~\ref{fig:pre1}. 

From Fig.~\ref{fig:pre1}, we identify several key observations. First, across both datasets, the overall ECE loss decreases after calibration, demonstrating the effectiveness of existing confidence calibration techniques at the global/graph level. However, despite this improvement, calibration failures persist within specific node groups. For instance, in the Pubmed dataset, the ECE loss of high-degree nodes increases after calibration rather than improving, suggesting that standard calibration techniques may be ineffective in correcting confidence misalignment in structurally important nodes. Similarly, in the Cora dataset, class-wise ECE increases, indicating that certain classes remain miscalibrated even after temperature scaling. Additionally, compared to the Cora dataset, the calibration effect on subgraphs in Pubmed is relatively minor, highlighting that certain graph subregions may be more susceptible to miscalibration. These findings suggest that calibration performance varies across different structural and semantic subgroups, and different datasets exhibit distinct sensitivities to these factors.

\begin{figure}[t]
  \centering
  \includegraphics[width=\linewidth]{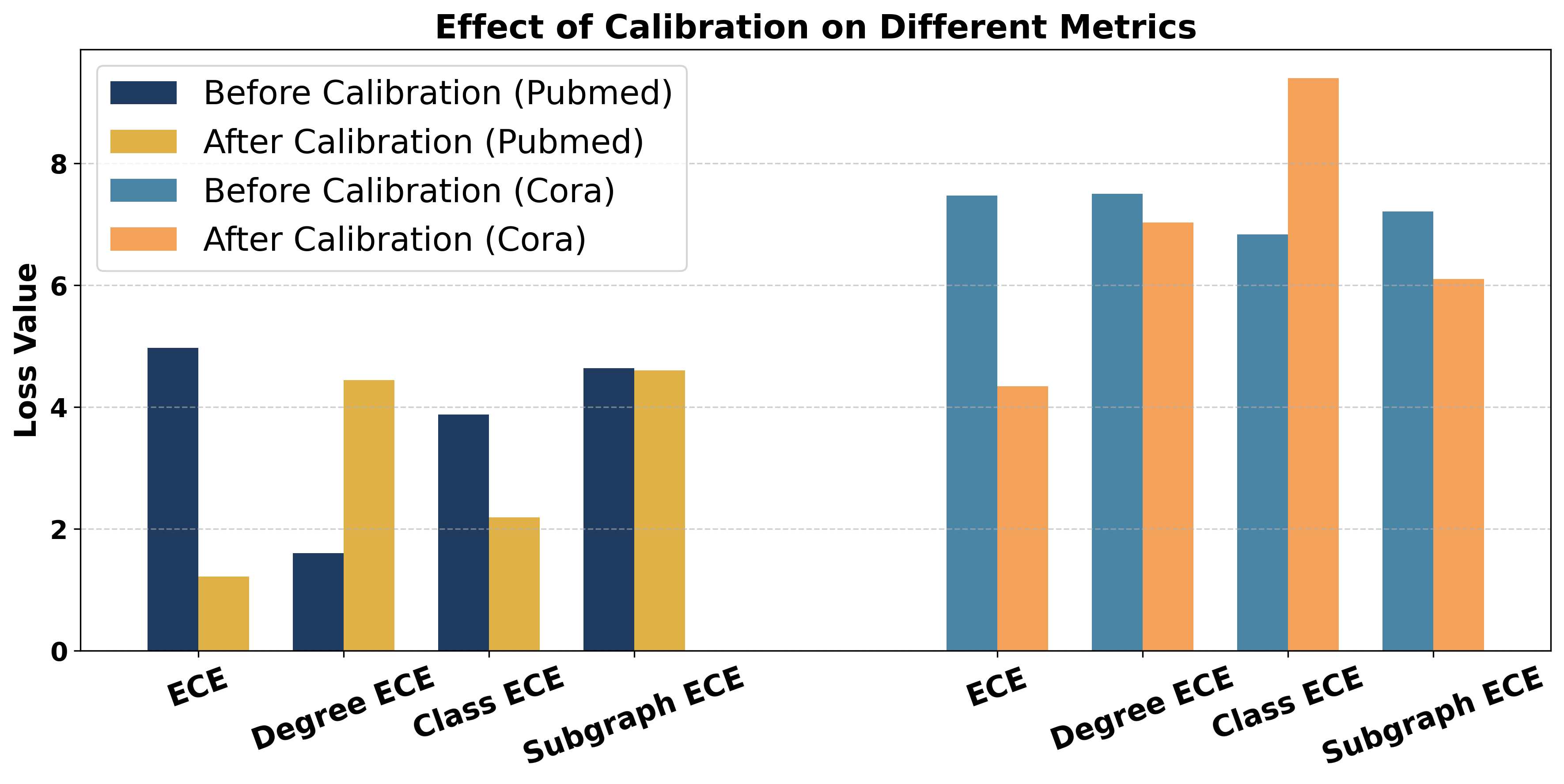}
  \vskip -1.2em
  \caption{Visualization of calibration performance on the Pubmed (left) and Cora (right)  datasets under different evaluation metrics. The y-axis represents the metric scores; lower scores indicate better calibration performance.
  }
  \Description{}
  \label{fig:pre1}
  \vskip -1em
\end{figure}

To further analyze these subgroup-specific calibration failures, we visualize reliability diagrams for both the entire graph and top-25\% high degree nodes on Pubmed with various calibration methods in Figure ~\ref{fig:pre2}. At the global level, GNN predictions exhibit an underconfident tendency, where predicted confidence scores are lower than actual accuracy. All calibration methods successfully increase model confidence, improving alignment between confidence and accuracy. However, within the high-degree node groups, we observe a contrasting pattern: The GNN prediction displays an overconfident tendency, where predicted confidence scores exceed actual accuracy at the high-confidence parts. Ideally, the calibration model should reduce confidence in these overconfident groups to achieve better alignment. However, instead of mitigating overconfidence, all these calibration methods further increase confidence, amplifying the miscalibration issue rather than correcting it. This failure suggests that current post-hoc calibration techniques struggle to adapt to subgroup-specific miscalibration patterns. 

One possible reason for this failure lies in the optimization objective used for calibration. Most existing graph calibration models rely on cross-entropy (CE) loss at the global level to guide temperature scaling. While CE loss effectively optimizes classification performance, it does not explicitly enforce confidence-accuracy alignment. Instead, it encourages the model to push the logit of the correct class higher, making predictions more confident rather than directly improving calibration. In cases where the model is already overconfident, CE loss can further amplify this misalignment, leading to calibration failure within certain node groups.

Our findings highlight two critical observations: \textbf{(i)} \textit{group-level calibration failures are prevalent in graph-structured data, and factors such as degree distribution, class, and node location significantly influence calibration performance}; and \textbf{(ii)} \textit{group-specific miscalibration patterns vary across datasets, meaning that a uniform calibration approach is insufficient}. Since different node groups exhibit distinct calibration tendencies, it is essential to first identify miscalibrated groups rather than applying a one-size-fits-all calibration strategy across the entire graph. By introducing group-specific calibration mechanisms, we can better capture intra-group variability, ensuring improved confidence-accuracy alignment both at the global level and within specific subgroups.

\begin{figure}[t]
  \centering
  \includegraphics[width=\linewidth]{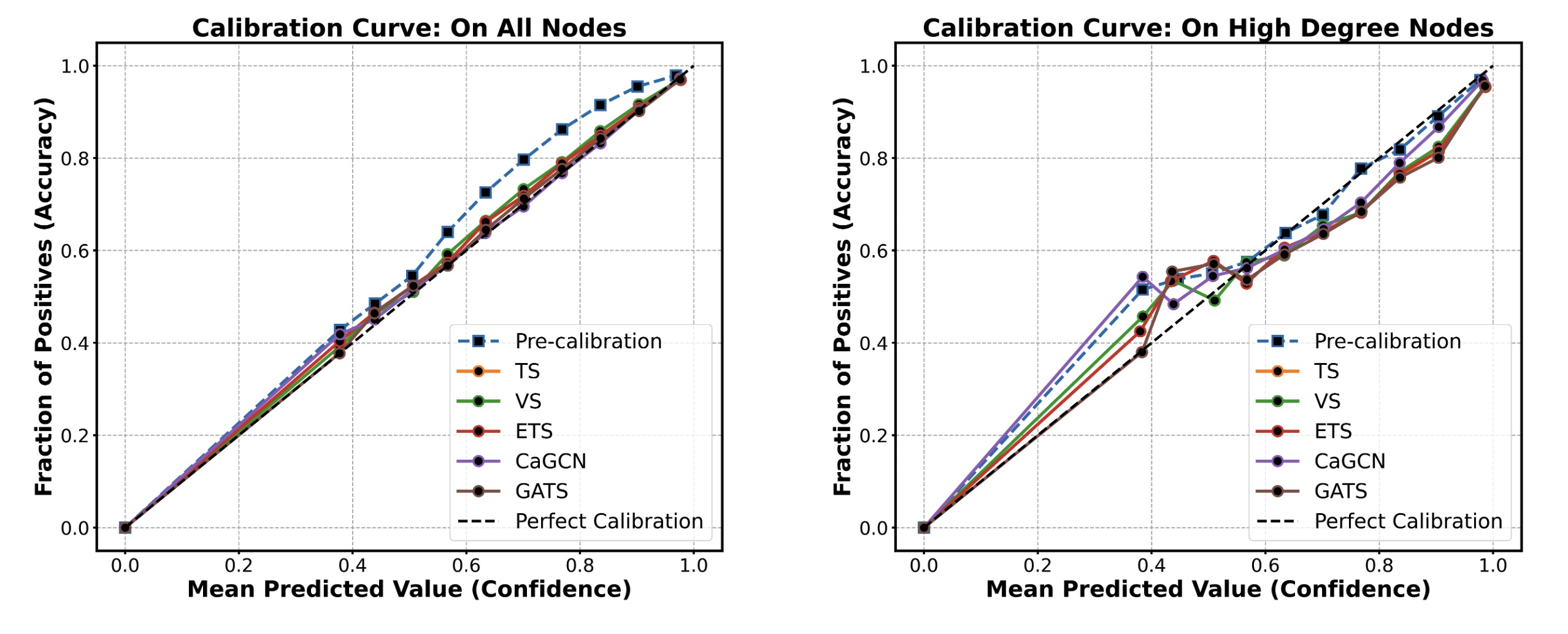}
  \vskip -1.2em
  \caption{Reliability diagrams of various calibration methods on the Pubmed dataset for the whole graph (left) and the top 25\% high-degree nodes (right). The y-axis denotes accuracy, and the x-axis denotes confidence. The diagonal line indicates perfect calibration, where confidence aligns exactly with accuracy. A curve above the diagonal means that accuracy exceeds confidence, indicating an under-confident model. Conversely, a curve below the diagonal implies accuracy is lower than confidence, meaning the model is overconfident.}
  \label{fig:pre2}
  \vskip -1em
\end{figure}

\section{Methodology}\label{sec:method}
Based on the findings from the above analysis, in this section, we propose a novel post-hoc calibration approach that adaptively refines calibration at both the global and subgroup levels to eliminate calibration bias across different node groups and achieve more effective confidence calibration. As illustrated in Figure~\ref{fig:ModelS}, our model consists of two key components: a calibration model \( f_{c} \) and a group detector \( f_{g} \). During the training process, the calibration model \( f_{c} \) receives the predicted logits from a pre-trained classification model and learns the node-wise temperature scaling, ensuring that the predicted probabilities better reflect true correctness likelihoods. In parallel, the group detector \( f_{g} \) is adversarially trained to automatically identify miscalibrated node groups, allowing the model to focus on regions where existing calibration methods fail. To jointly optimize classification and calibration, we incorporate two loss terms: the cross-entropy loss \( \mathcal{L}_{\mathrm{CE}} \), which aligns the logits with real-world data distributions to improve classification accuracy, and the group-wise differentiable ECE loss \( \mathcal{L}_{\mathrm{ECE-group}} \), which explicitly refines calibration by enforcing better confidence-accuracy alignment within the detected miscalibrated groups. By integrating these components, our model dynamically adapts to dataset-specific miscalibration patterns, effectively addressing the limitations of existing calibration methods. We introduce the details of each component in the following sections.

\subsection{Temperature Scaling Learning}
Generally, in post-hoc confidence calibration for node classification on graphs, the input consists of the adjacency matrix \( \mathbf{A} \in \mathbb{R}^{N \times N} \) and the logits \( \mathbf{Z} \in \mathbb{R}^{N \times C} \) obtained from a pre-trained classifier, where \( N \) represents the number of nodes and \( C \) denotes the number of classes. The calibration model \( f_c \) learns a node-wise temperature scaling function, producing a temperature vector \( \mathbf{t} \in \mathbb{R}^{N \times 1} \) that adjusts the confidence scores for each node independently. The temperature is computed as:
\begin{equation}\label{calimodel}
    \mathbf{t} = f_{c}(\mathbf{A}, \mathbf{Z}; \theta_c),
\end{equation}
where \( \theta_c \) is the learnable parameters of the calibration model \( f_c \). Various GNN models can be used as \( f_c \), such as GCN~\cite{kipf2017GCN} and GAT~\cite{velickovic2018GAT}. After computing the node-wise temperature, the temperature-scaled logits $\hat{\mathbf{Z}}$ and calibrated probabilities $\hat{\mathbf{P}}$ are given as:
\begin{equation}
    \hat{\mathbf{Z}}_i = \frac{\mathbf{Z}_i}{t_i}, \quad \hat{\mathbf{P}}_i = \mathrm{Softmax}(\hat{\mathbf{Z}}_i), \quad \forall v_i \in \mathcal{V}.
\end{equation}
where \( \hat{\mathbf{Z}}_{i} \) denotes the temperature-scaled logits for node \( v_i \), \( t_{i} \) is the temperature parameter for $v_i$, $\mathbf{Z}_i$ is the original logits, and $\hat{\mathbf{P}}_i$ denotes the calibrated probability distribution after applying the softmax function. As noted in Section~\ref{sec:background}, the standard ECE relies on binning strategies to estimate miscalibration, making it non-differentiable and unsuitable for direct optimization in training. Thus, instead of directly using ECE as the loss to train the model, existing works ~\cite{tang2024simcalib, hsu2022makes, yang2024calibrating} minimize cross-entropy loss $\mathcal{L}_{\mathrm{CE}}$ based on the calibrated probabilities to learn the model as:
\begin{equation}
    \mathcal{L}_{\mathrm{CE}} := -\sum_{v_i \in \mathcal{V}_L} \sum_{c=1}^C y_{i,c} \log \hat{P}_{i,c},
\end{equation}
where \( \mathcal{V}_L \) denotes the set of labeled nodes, \( y_{i,c} \) is the ground-truth label in one-hot encoding, and \( \hat{P}_{i,c} \) represents the calibrated probability of node \( v_i \) belonging to class \( c \). 

\begin{figure}[!t]
  \centering
  \includegraphics[width=0.48\textwidth]{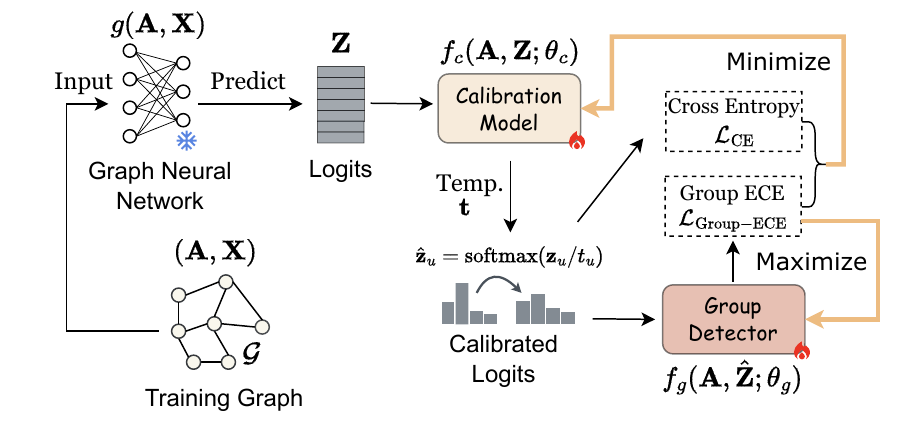} 
  \vskip -1.2em
  \caption{Illustration of the model structure.} 
  \Description{}
  \label{fig:ModelS}
  \vskip -1em
\end{figure}

\subsection{Mis-calbirated Group Detection}

While $\mathcal{L}_{\mathrm{CE}}$ helps reduce overall expected calibration error, it does not explicitly address subgroup-level calibration, potentially leading to calibration failures within certain subgroups as demonstrated in our preliminary experiments in Section~\ref{PreliminaryAnalysis}. To solve this issue, we propose to explicitly calibrate confidence for groups. However, as shown in Section~\ref{PreliminaryAnalysis}, subgroup-specific mis-calibration patterns vary across datasets. It is unclear for the given dataset and the target classifier, which groups would fail. In addition, if we predefine some groups based on some factors, e.g., degrees, though we could have better calibration results on predefined groups by explicitly optimizing their calibration, it might result in bias or failure on groups based on another factor, e.g., subclasses. To address these challenges, we propose to adopt a group detector $f_g$ to dynamically detect groups that the calibration model does not work well during training, and propose a novel group-based ECE to optimize them. 

Specifically, we use a two-layer Graph Isomorphism Network (GIN)~\cite{xu2018powerful} followed by a linear mapping as the group detector $f_g$ to identify miscalibrated node groups in the graph. The group detector takes the adjacency matrix \( \mathbf{A} \) and the temperature-scaled logits \( \hat{\mathbf{Z}} \) as input to a two-layer GIN to learn node representation $\mathbf{H}$ as:
\begin{equation}
    \mathbf{H} = \mathrm{GIN}(\mathbf{A},\hat{\mathbf{Z}}).
\end{equation}
By leveraging a GIN-based architecture, the group detector effectively captures both node attributes and structural dependencies, ensuring that group assignment is informed by both local graph connectivity and confidence calibration trends. After obtaining the node representation matrix \( \mathbf{H}\), we apply a linear mapping followed by a softmax function to compute the final group weight matrix as:
\begin{equation} \label{eq:group_detector}
    \mathbf{G} = \mathrm{Softmax} (\mathbf{H}\mathbf{W}) \in \mathbb{R}^{N \times K},
\end{equation}
where \( \mathbf{W} \in \mathbb{R}^{K \times d} \) is a learnable weight matrix that projects the hidden representations into the group assignment space. Each entry \( G_{i,k} \) in the group weight matrix represents the weight of node \( v_i \) belonging to the group \( k \), and \( K \) is the pre-defined number of groups. The softmax function ensures that the group assignments $\mathbf{G}_i$ for each node $v_i$ sum to 1, i.e., $\sum_{k=1}^{K} G_{i,k} = 1, \forall v_{i} \in \mathcal{V}.$
The output matrix \( \mathbf{G} \) provides a soft assignment of each node to multiple groups rather than a hard clustering, allowing the model to flexibly capture miscalibration patterns across different node subpopulations. 
The group detector will be trained to detect the groups that could result in bad ECE. The detected groups will be used to apply differentiated calibration adjustments, which are further refined by our proposed group-wise differentiable ECE loss, described in the next section.

\subsection{Group-based ECE}\label{sec:gb_ece}

Existing calibration models rely solely on cross-entropy loss to guide temperature scaling, which does not explicitly enforce confidence-accuracy alignment, resulting in subgroup calibration failure. To address this limitation, we design a novel group-wise differentiable ECE loss based on the group assignments provided by the group detector. The loss is defined as:

\begin{equation}\label{eq:groupece}
\mathcal{L}_{\mathrm{Group\text{-}ECE}} := \sum_{j=1}^K \frac{\mathbf{G}_{:, j}^\top\cdot\mathbf{1}_N}{|\mathcal{V}|} \cdot ( \mathrm{SoftAcc}(\mathbf{G}_{:, j}) - \mathrm{SoftConf}(\mathbf{G}_{:, j}))^2
\end{equation}
where $\mathbf{G}_{:, j}\in\mathbb{R}^N$ is the $j$-th column in the group weight matrix $\mathbf{G}\in\mathbb{R}^{N\times K}$, and $\mathbf{1}_N\in\mathbb{R}^N$ is a one-vector. Specifically, we define $\mathrm{SoftAcc}(\mathbf{G}_{:, j})$ and $\mathrm{SoftConf}(\mathbf{G}_{:, j})$ as the mean accuracy and mean confidence within soft group $j$, respectively, and measure the calibration error using the squared error between accuracy and confidence. The mean accuracy and mean confidence in a soft group are given as:
\begin{align}
    \mathrm{SoftAcc}(\mathbf{G}_{:, j})&:=~(\mathbf{G}_{:, j}^\top\cdot\mathbf{1}_N)^{-1} \sum_{v_i\in\mathcal{V}_L} G_{i,j}\cdot \mathbbm{1} (y_i = \hat{y}_i), \label{eq:soft_acc}\\
    \mathrm{SoftConf}(\mathbf{G}_{:, j})&:=~(\mathbf{G}_{:, j}^\top\cdot\mathbf{1}_N)^{-1} \sum_{v_i\in\mathcal{V}_L} G_{i,j}\cdot \hat{c}_i. \label{eq:soft_conf}
\end{align}
The confidence score of node \( v_i \) is $\hat{c}_i := \max_{c} \hat{P}_{i,c}$, where \( \hat{P}_{i,c} \) is the calibrated probability of node \( v_i \) belonging to class \( c \) learned by the calibration model. The indicator function \( \mathbbm{1}(\cdot) \) equals 1 if the predicted label \( \hat{y}_i = \arg\max_c \hat{P}_{i,c} \) matches the ground-truth label \( y_i \), and 0 otherwise. The loss is computed only on labeled nodes \( \mathcal{V}_L \), ensuring that calibration is evaluated in a semi-supervised setting.

The proposed Group-ECE objective in Eq.~\ref{eq:groupece} offers two key advantages: differentiability and improved calibration performance. Unlike the traditional ECE loss, which relies on fixed binning strategies and is non-differentiable, our formulation computes the calibration error using soft assignment weights across multiple groups. This enables direct optimization via gradient descent, enhancing both adaptability and effectiveness. Moreover, our method can be viewed as a squared error variant of the original ECE loss in Definition~\ref{def:orig_ece}, leading to better overall performance compared to the original absolute error formulation (see Section~\ref{Ablation} for the comparison results). Importantly, our formulation remains lossless: when the squared error is replaced by the absolute error, the conventional ECE loss is exactly recovered as a special case of our framework. This relationship is formally established in Appendix~\ref{sec:append_proof4}.

\subsection{Final Objective Function}
The Group-ECE loss reflects the calibration performance within each detected group. Since the role of the group detector is to identify miscalibrated node groups in the graph, we employ an adversarial training strategy where the group detector maximizes the Group-ECE loss to highlight miscalibrated node groups, where confidence misalignment remains high even after calibration, ensuring that the calibration model prioritizes these regions during optimization.  
On the other hand, the calibration model is responsible for learning appropriate temperature scaling to correct confidence misalignment. Thus, during training, we minimize the Group-ECE loss with respect to the calibration model’s parameters \( \theta_c \) to refine confidence estimation and reduce miscalibration. The overall training objective is formulated as:
\begin{equation}\label{finalobjective}
\min_{\theta_{c}} \max_{\theta_{g}} \left(  \mathcal{L}_{CE}\left(\theta_{c}\right) + \lambda \cdot \mathcal{L}_{\mathrm{Group\text{-}ECE}}(\theta_{c}, \theta_{g})\right),    
\end{equation}
where $\lambda$ is a scalar to control the contribution of Group ECE loss. 

During training, the \( \max_{\theta_{g}} \) process forces the group detector to assign higher weights to node groups exhibiting significant miscalibration by maximizing the Group-ECE loss, while \( \min_{\theta_{c}} \) optimizes the calibration model to minimize the calibration error within the detected groups, improving the overall confidence-accuracy alignment. This adversarial optimization framework ensures that the group detector and the calibration model are jointly optimized, enabling the detector to effectively identify miscalibrated regions and allowing the calibration model to focus on correcting miscalibration within these groups. By dynamically adjusting to dataset-specific miscalibration patterns, our approach achieves more robust and adaptive confidence calibration across different node groups. The training algorithm and time complexity analysis of our model are presented in Appendix.~\ref{app:timecomplexity}.

\begin{table*}[t]
\centering
\caption{Calibration performance comparison in terms of Global ECE. The best performance is highlighted in boldface.}
\label{tab:GlobalECE}
\vskip -1em
\resizebox{0.9\textwidth}{!}{%
\begin{tabular}{lcccccccc}
\toprule
Methods & Cora & Citeseer & Pubmed & Photo & Computer & CS & Physics & CoraFull \\
\midrule
Uncali & 7.47 $\pm$ 3.27  & 9.02 $\pm$ 4.92 & 4.97 $\pm$ 0.65 & 2.01 $\pm$ 0.35 & 3.69 $\pm$ 0.44 & 2.60 $\pm$ 0.47 & 1.70 $\pm$ 1.00 & 7.77 $\pm$ 0.34  \\
TS & 5.42 $\pm$ 0.52 & 8.23 $\pm$ 0.45 & 1.30 $\pm$ 0.22 & 2.04 $\pm$ 0.31 & 2.85 $\pm$ 0.32 & 2.05 $\pm$ 0.16 & 1.56 $\pm$ 0.29 & 6.92 $\pm$ 0.52\\
VS & 5.98 $\pm$ 0.38 & 7.63 $\pm$ 0.55 & 1.26 $\pm$ 0.18 & 1.74 $\pm$ 0.33 & 2.64 $\pm$ 0.21 & 2.05 $\pm$ 0.16 & 1.55 $\pm$ 0.21 & 6.42 $\pm$ 0.26 \\
CaGCN & 5.90 $\pm$ 1.50 & 7.80 $\pm$ 0.67 & 1.14 $\pm$ 0.25 & 1.76 $\pm$ 0.33 & 1.91 $\pm$ 0.66 & 1.58 $\pm$ 0.45 & 0.92 $\pm$ 0.43  & 6.85 $\pm$ 0.19  \\
DCGC & 4.23 $\pm$ 0.45 & 4.50 $\pm$ 0.21 & 1.18 $\pm$ 0.16 & 1.60 $\pm$ 0.48 & 2.01 $\pm$ 0.38 & 1.84 $\pm$ 0.31 & 0.60 $\pm$ 0.18 & 4.25 $\pm$ 0.65 \\
SimCali & 4.15 $\pm$ 0.98 & 4.59 $\pm$ 0.54 & 1.14 $\pm$ 0.15 & 1.61 $\pm$ 0.38 & 1.63 $\pm$ 0.42 & 1.56 $\pm$ 0.31 & 0.56 $\pm$ 0.16 & 4.58 $\pm$ 0.45\\
GATS &  4.34 $\pm$ 1.64  &  4.73 $\pm$ 0.87 & 1.22 $\pm$ 0.40 & 1.63 $\pm$ 0.42 & 2.48 $\pm$ 0.52 & 1.51 $\pm$ 0.26 & 0.60 $\pm$ 0.11 & 4.65 $\pm$ 0.44\\ 
AdvCali & \textbf{3.59} $\pm$ 0.63 & \textbf{4.39} $\pm$ 0.20 & \textbf{1.10} $\pm$ 0.48 & \textbf{1.56} $\pm$ 0.54 & \textbf{1.53 $\pm$ 0.27} & \textbf{1.44 $\pm$ 0.18} & \textbf{0.53 $\pm$ 0.09} & \textbf{3.95 $\pm$ 1.06} \\
\bottomrule
\end{tabular}%
}
% \vskip -1em % Adjust space after the table
\end{table*}

\begin{table*}[t]
\centering
\caption{Calibration performance in terms of Degree ECE and Class-wise ECE}
\label{tab:GroupECE}
\vskip -1em
\resizebox{1.0\textwidth}{!}{%
\begin{tabular}{l l c c c c c c c c}
\toprule
Metrics & Methods & Cora & Citeseer & Pubmed & Photo & Computer & CS & Physics & CoraFull \\
\midrule
\multirow{7}{*}{Degree ECE} 
& Uncali & 7.50 ± 0.95 & 9.21 ± 4.17 & 1.60 ± 0.43 & 1.49 ± 0.69 & 2.25 ± 0.56 & 2.94 ± 0.88 & 0.98 ± 0.35 & 6.67 ± 0.51\\
& DCGC & 8.45 ± 0.64 & 7.33 ± 0.76 & 3.84 ± 0.31 & 1.94 ± 0.81 & 3.24 ± 0.56 & 2.19 ± 0.42 & 0.94 ± 0.18 & 7.02 ± 0.56\\
& SimCali & 6.95 ± 0.54 & 6.93 ± 0.64 & 3.57 ± 0.54 & 1.84 ± 0.62 & 3.05 ± 0.39 & 2.06 ± 0.38 & \textbf{0.84 ± 0.18} & 6.82 ± 0.74\\
& CaGCN &10.38 ± 1.14 & 9.13 ± 0.72 & 3.04 ± 1.15 & 2.00 ± 0.96 & 2.91 ± 0.83 & 1.98 ± 0.69& 0.85 ± 0.32& 11.20 ± 4.49\\
& GATS & 7.03 ± 0.85 & 6.43 ± 0.84 & 4.44 ± 0.25 & \textbf{1.83 ± 1.00} & 3.18 ± 0.87 & 2.27 ± 0.38 & 1.07 ± 0.18 & 7.48 ± 0.60 \\
& AdvCali & \textbf{6.16 ± 0.95} & \textbf{6.35 ± 1.71} & \textbf{2.57 ± 1.12} & 1.86 ± 0.93 & \textbf{2.78 ± 0.98} & \textbf{1.92 ± 0.65} & 0.89 ± 0.14 & \textbf{6.03 ± 0.57} \\
\midrule
\multirow{7}{*}{Class ECE} 
& Uncali & 2.79$\pm$1.38 & 4.04$\pm$2.32 & 3.48$\pm$0.37 & 0.80$\pm$0.09 & 0.96$\pm$0.13 & 0.52$\pm$0.06 & 0.85$\pm$0.37 & 0.38$\pm$0.01\\
& DCGC & 2.31$\pm$0.45 &3.29$\pm$0.38 & 1.21$\pm$0.19 & 0.78$\pm$0.10  & 0.73$\pm$0.05 & 0.48$\pm$0.09 & 0.58$\pm$0.09 & 0.42$\pm$0.03\\
& SimCali & 2.13$\pm$0.44 &3.10$\pm$0.51 & 1.24$\pm$0.23 & 0.75$\pm$0.21  & 0.72$\pm$0.10  & 0.49$\pm$0.04 & 0.56$\pm$0.10 & 0.37$\pm$0.02\\
& CaGCN &2.49$\pm$0.43 &3.54$\pm$0.44 & 1.20$\pm$0.16 & 0.78$\pm$0.08  & 0.76$\pm$0.04  & 0.49$\pm$0.06 & 0.57$\pm$0.07 & 0.40$\pm$0.06\\
& GATS & 2.05$\pm$0.26 & \textbf{3.05$\pm$0.59} & 1.26$\pm$0.22 & 0.74$\pm$0.07 & 0.77$\pm$0.06 & 0.48$\pm$0.09 & 0.55$\pm$0.13 & 0.39$\pm$0.01 \\
& AdvCali & \textbf{1.91$\pm$0.08} & 3.08$\pm$0.59  & \textbf{1.19$\pm$0.30} & \textbf{0.72$\pm$0.08} & \textbf{0.69$\pm$0.09} & \textbf{0.46$\pm$0.07 } & \textbf{0.53$\pm$0.12} & \textbf{0.35$\pm$0.01}  \\
\bottomrule
\end{tabular}%
}
% \vskip -1em % Adjust space after the table
\end{table*}

\section{Experiments}

In this section, we conduct comprehensive experiments to evaluate the effectiveness of our proposed framework. Specifically, we aim to address the following research questions: \textbf{RQ1}: How well does our framework improve confidence calibration across different benchmark datasets and different node groups? \textbf{RQ2}: Can our model effectively adapt to different calibration and classifier backbones? \textbf{RQ3}: What insights can be derived from the learned group assignments? Can the group detector capture dataset-specific factors that influence calibration performance?

\subsection{Experimental Setup}

\noindent\textbf{Baselines}. For a comprehensive evaluation, we benchmark our proposed model, AdvCali, using GATS~\cite{hsu2022makes} as the calibration backbone and compare it against several state-of-the-art calibration methods. The baselines included in our experiments are as follows: \textbf{TS}\cite{guo2017calibration}, which applies a single temperature parameter to adjust confidence scores across all classes; \textbf{VS}\cite{guo2017calibration}, an extension of TS that learns class-wise scaling parameters for more refined calibration; \textbf{SimCali}\cite{tang2024simcalib}, which incorporates node-wise similarity scores to enhance calibration; \textbf{DCGC}\cite{yang2024calibrating}, which improves calibration from a data-centric perspective by reinforcing the edge weights of homophilous and structurally significant edges; \textbf{CaGCN}\cite{wang2021confident}, which encourages higher confidence for correctly predicted samples; and \textbf{GATS}\cite{hsu2022makes}, which accounts for confidence differences within ego-graphs and incorporates distance to the training set into calibration.
\noindent\textbf{Datasets}. We conduct experiments on eight widely used real-world graph benchmarks: \textbf{Cora}~\cite{yang2016revisiting}, \textbf{Citeseer}~\cite{yang2016revisiting}, \textbf{Pubmed}~\cite{yang2016revisiting}, \textbf{CS}~\cite{shchur2018pitfalls}, \textbf{Photo}~\cite{shchur2018pitfalls}, \textbf{Computers}~\cite{shchur2018pitfalls}, \textbf{Physics}~\cite{shchur2018pitfalls}, and \textbf{CoraFull}~\cite{bojchevski2018deep}. 
More dataset details are in Appendix.~\ref{sec:append_dataset}.

\noindent\textbf{Evaluation Metrics}.  To comprehensively assess the calibration performance of different methods, we employ three evaluation metrics: Global ECE, Degree ECE, and Class-wise ECE. (i) Global ECE: This metric evaluates the overall calibration performance by computing the expected calibration error (ECE) loss across all nodes in the graph, treating the entire graph as a single entity. (ii) Degree ECE: To measure how well models calibrate high-degree nodes, we compute the ECE loss specifically on the top 25\% highest-degree nodes in the graph. (iii) Class-wise ECE: This metric calculates the ECE loss separately for each class and then takes the average across all classes. For all three ECE metrics, we adopt a binning strategy with a fixed number of bins $M=15$, ensuring a consistent calibration evaluation across methods. 

\noindent\textbf{Implementation}. Following standard protocols in work \cite{hsu2022makes}, we employ a 2-layer GCN classifier as the backbone for node classification on each dataset. The logits obtained from the pre-trained GCN are then passed to each calibration model for post-hoc confidence calibration.  For all experiments, we perform five independent runs with different random splits, where the labeled/unlabeled data ratio is set to 15/85. Additionally, we apply a three-fold internal cross-validation on the labeled data to optimize the hyperparameters of the calibration models.

\subsection{Calibration Performance}
To address \textbf{RQ1}, we conduct each experiment 5 times and report the averaged Global ECE with standard deviations in Table~\ref{tab:GlobalECE}, and averaged Degree ECE and Class-wise ECE in Table~\ref{tab:GroupECE}. Table~\ref{tab:GlobalECE} focuses on evaluating the overall calibration performance by computing Global ECE. Table~\ref{tab:GroupECE} aims to explore the subgroup-level calibration performance of our model. Specifically, we report two key metrics: Degree ECE, and Class ECE. From the two tables, we observe:

\begin{itemize}[leftmargin=*]
    \item Compared to existing calibration models, our proposed AdvCali model achieves the lowest ECE loss across all datasets at the global level and surpasses 13 out of 16 state-of-the-art methods in group-level calibration. This demonstrates that our model not only improves calibration at the global level but also effectively reduces calibration bias among different node groups, thereby enhancing subgroup-level calibration performance.
    
    \item While previous methods, such as CaGCN and GATS, improve global calibration, they fail to calibrate different node subgroups equally. In contrast, our AdvCali explicitly optimizes subgroup-level calibration by dynamically identifying miscalibrated groups and refining their confidence estimates. This validates the effectiveness of our framework and underscores the importance of subgroup-aware calibration in graph-based learning tasks.
    
    \item Although our model achieves strong results, we observe that on certain datasets, such as CS and CoraFull, existing calibration methods have limited effectiveness in optimizing class-wise ECE. Additionally, in Pubmed, Computer, and Physics datasets, we find that while calibrated logits exhibit improved confidence estimation at the global level, high-degree nodes still suffer from greater confidence estimation errors compared to uncalibrated models. These observations highlight the widespread existence of group-level miscalibration in graph datasets, posing significant challenges for achieving reliable GNN predictions.
\end{itemize}

\subsection{Flexibility of AdvCali to Various Backbones}
To address \textbf{RQ2} and investigate whether our proposed model can effectively adapt to different calibration and classifier backbones, we replace the original pre-trained GCN classifier with GraphSAGE. Using the logits obtained from GraphSAGE, we evaluate our AdvCali model with two different $f_{c}$: GCN \cite{kipf2017GCN} and GATS \cite{hsu2022makes}. The results, presented in Table~\ref{tab:global_ece_graphsage}, report the Global ECE performance under different backbone configurations. From Table~\ref{tab:global_ece_graphsage}, we observe the following key insights: (i) Our proposed AdvCali model consistently improves calibration performance across different classifier backbones, demonstrating its generalizability. The performance gains achieved when switching from GCN to GraphSAGE indicate that our method is not restricted to a specific feature aggregation mechanism. (ii) Regardless of whether GCN or GATS is used as the calibration backbone, AdvCali consistently outperforms the corresponding baseline models, reaffirming its flexibility in adapting to different post-hoc calibration approaches.

\begin{table*}[t]
\centering
\caption{Global ECE performance across different datasets using GraphSAGE as the classifier.}
\vskip -1em
\label{tab:global_ece_graphsage}
\resizebox{\textwidth}{!}{%
\begin{tabular}{lccccccccc}
\toprule
Classifier & Model & Cora & Citeseer & Pubmed & CS & Photo & Computer & Physics & CoraFull \\
\midrule
\multirow{5}{*}{GraphSAGE} 
& Uncali & 7.36$\pm$2.98 & 5.80$\pm$1.14 & 1.26$\pm$0.21 & 3.35$\pm$0.17 & 2.03$\pm$1.32 & 2.98$\pm$1.03 & 1.72$\pm$0.59 & 8.27$\pm$0.95 \\
& GCN & 6.02$\pm$1.51 & 5.59$\pm$1.38 & 1.25$\pm$0.42 & 2.00$\pm$0.39 & 2.00$\pm$0.39 & 2.40$\pm$0.62 & 1.24$\pm$0.28 & 7.92$\pm$0.43 \\
& GATS & 3.84$\pm$0.95 & 4.85$\pm$0.92 & 1.22$\pm$0.54 & 3.00$\pm$0.56 & 1.78$\pm$0.81 & 2.73$\pm$0.88 & 0.93$\pm$0.19 & 7.69$\pm$0.63 \\
& AdvCali(GCN) & 5.35$\pm$0.96 & 5.17$\pm$0.74 & 1.21$\pm$0.32 & \textbf{1.98$\pm$0.26} & 1.57$\pm$0.62 & \textbf{2.14$\pm$0.25} & 0.99$\pm$0.22 & 6.35$\pm$0.76 \\
& AdvCali(GATS) & \textbf{3.70$\pm$1.11} & \textbf{4.76$\pm$0.93} & \textbf{1.19$\pm$0.31} & 2.39$\pm$0.66 & \textbf{1.48$\pm$0.54} & 2.40$\pm$0.59 & \textbf{0.77$\pm$0.23} & \textbf{4.89$\pm$0.60} \\
\bottomrule
\end{tabular}%
}
\end{table*}

\subsection{Ablation Study}\label{Ablation}
To elucidate the contribution of each component to the efficacy of our AdvCali, we examine various variants of \textbf{AdvCali}: (i) \textbf{AdvCali w/o CE}: Our proposed model trained using only the Group-ECE loss, without the cross-entropy loss. (ii) \textbf{AdvCali w/o Group-ECE}: Our proposed model trained using only the cross-entropy loss, without the Group-ECE loss. (iii) \textbf{AdvCali Min}: A variant where we minimize the Group-ECE loss instead of employing adversarial training for the group detector.(iv)  \textbf{AdvCali Absolute}: Replace the squared error in Equation \ref{eq:groupece} with absolute error. 

We present the Global ECE results of these variants on Cora and Citeseer in Figure~\ref{fig:abl}. From the figure, we can observe: (i) AdvCali achieves the lowest ECE loss, outperforming both AdvCali w/o Group-ECE and AdvCali w/o CE. This demonstrates that both loss terms play crucial roles in improving confidence calibration. The combination of cross-entropy loss and Group-ECE loss ensures that the model not only learns a well-calibrated confidence distribution globally but also refines subgroup-level calibration. (ii) AdvCali w/o CE performs worse than AdvCali w/o Group-ECE, suggesting that relying solely on Group-ECE is insufficient to achieve optimal calibration. While the Group-ECE loss effectively reduces subgroup miscalibration, it does not provide a strong enough supervisory signal for classification performance, leading to suboptimal results. (iii) AdvCali Min underperforms the full AdvCali model, highlighting the importance of adversarial training. (vi) Replacing squared error with absolute error leads to a noticeable performance drop, indicating that squared error provides more stable optimization and better retains critical information. 
The adversarial training allows the detector to identify critical regions where confidence misalignment is most severe, enabling more effective recalibration.

\begin{figure}[t]
  \centering
  \includegraphics[width=0.8\linewidth]{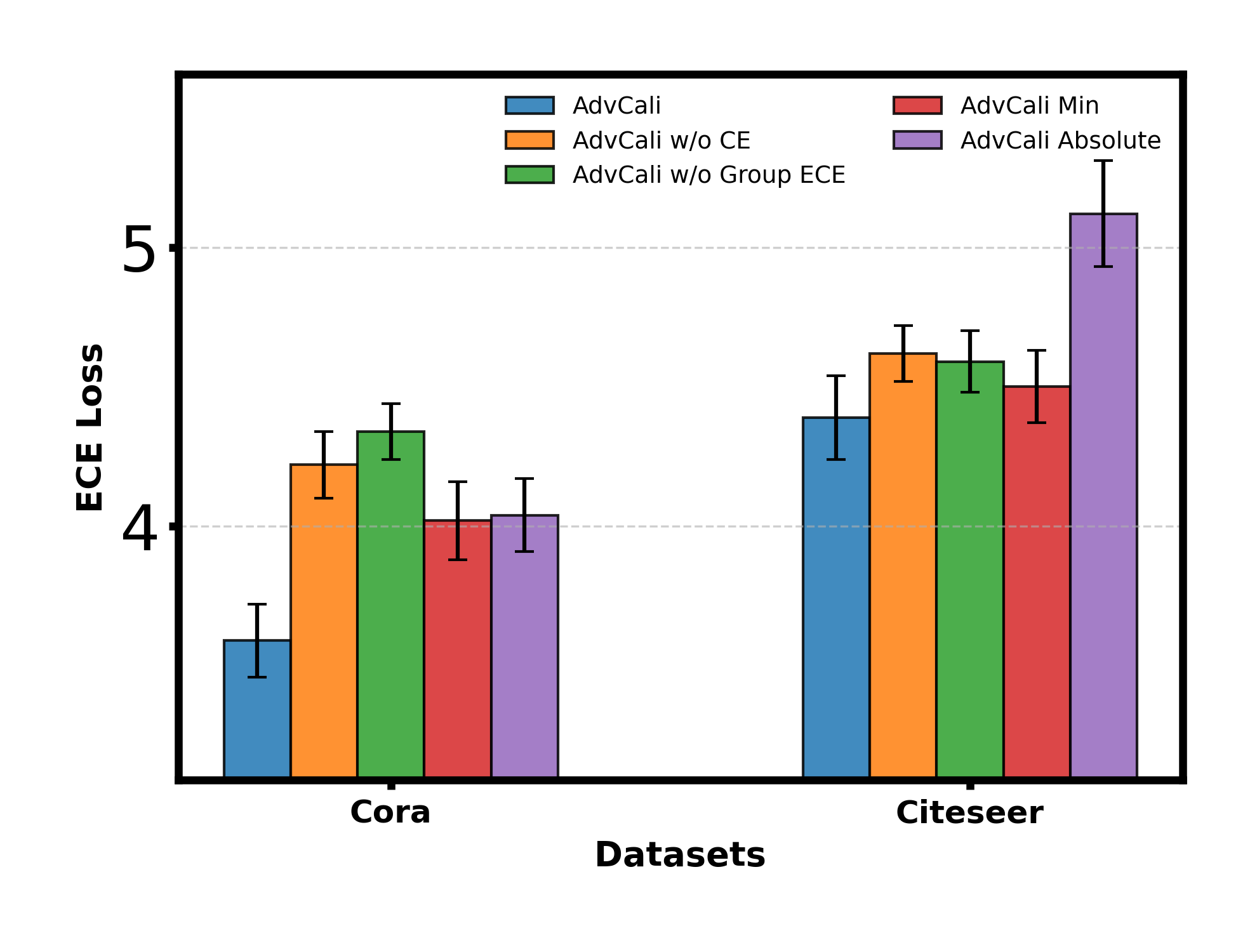}
  \vskip -2.5em
  \caption{Ablation study results.}
  \label{fig:abl}
  \vskip -1em
\end{figure}

\subsection{Hyper-parameter Analysis}

In this section, we further investigate the impact of hyperparameters on our AdvCali. Specifically, we focus on two key hyperparameters: \( \lambda \), which controls the contribution of the Group-ECE loss, and \( K \), which defines the pre-determined number of groups in the group detector. To assess the sensitivity of these hyperparameters, we vary \( \lambda \) over the set \(\{0.1, 1, 10, 100\}\) on CS and Photo, and \( K \) over the set \(\{4, 8, 12, 16\}\) on Computer and Pubmed. The results are given in Figure~\ref{fig:hyper}. From Figure~\ref{fig:hyper}, we derive the following observations: (i) As \( \lambda \) increases, the model's calibration loss initially decreases, indicating that incorporating Group-ECE loss effectively enhances confidence calibration. However, beyond a certain threshold, performance begins to degrade, suggesting that overemphasizing the Group-ECE loss can lead to suboptimal calibration. This highlights the importance of balancing the contributions of cross-entropy loss and Group-ECE loss in the optimization objective. (ii) The model's performance remains relatively stable across different values of \( K \), suggesting that the group detector is robust to variations in the number of pre-defined groups. This indicates that our model can effectively adapt to different granularities of miscalibration without requiring precise tuning of \( K \).

\begin{figure}[t]
  \centering
  \includegraphics[width=\linewidth]{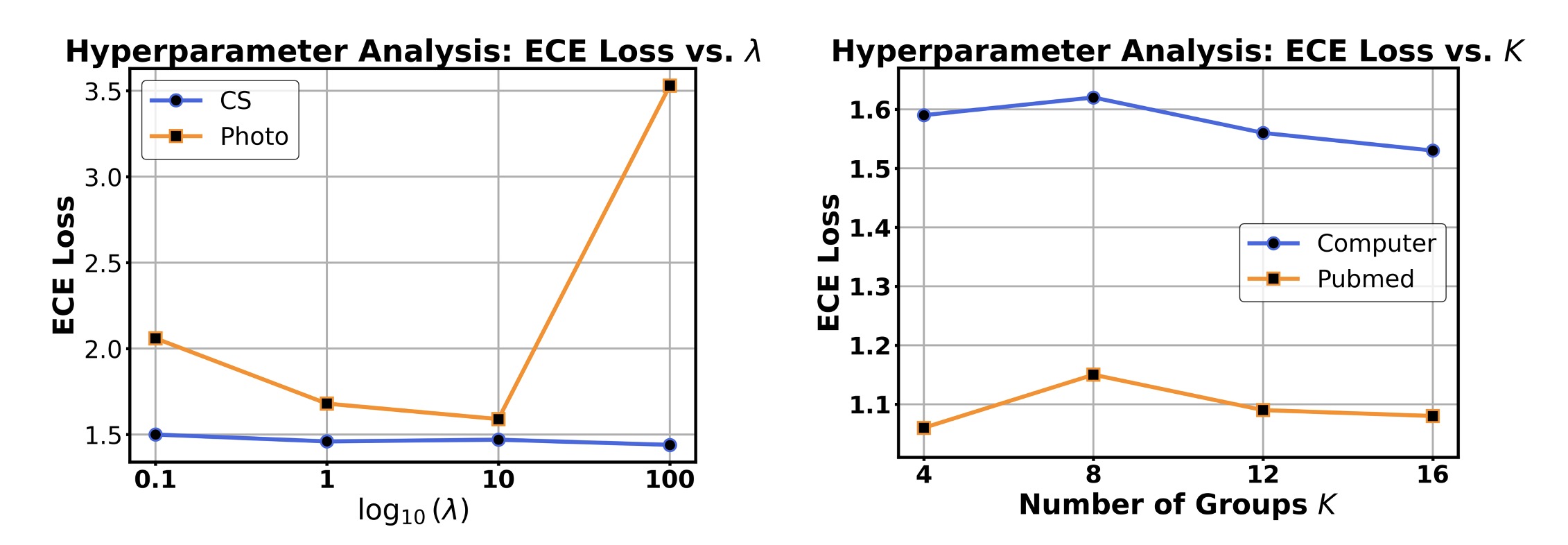}
  \vskip -1.2em
  \caption{Hyperparameter analysis of the Group-ECE contribution factor $\lambda$ and the number of groups $K$. 
  }
  \Description{}
  \label{fig:hyper}
  \vskip -1em
\end{figure}

\subsection{Understanding Group Detector}

To address \textbf{RQ3}, we examine whether the group detector can automatically differentiate nodes based on dataset-specific factors affecting calibration, even though these factors are not explicitly provided as input. We conduct experiments using our AdvCali(GATS) on three datasets: Cora, where node degree does not strongly influence calibration but certain class-specific nodes (e.g., Class 0) exhibit calibration failures, and Pubmed and Computer, where high-degree nodes face significant calibration challenges.

To analyze the learned group of the group detector, we visualize two key metrics during training: (i) \textit{Degree Standard Deviation Across Groups}, quantifies how node degree varies across the learned groups. At each epoch, we determine each node's assigned group label as $\arg\max_k G_{i,k}$.  
Then we compute the mean node degree for each group, forming a vector $\mathbf{d} = [d_0, d_1, ..., d_{K-1}]$, where $d_k$ is the mean degree of nodes assigned to group $k$. We then track $\operatorname{std}(\mathbf{d})$ over training epochs, where a higher standard deviation indicates greater differentiation between groups based on degree; and (ii) \textit{Class 0 Distribution Standard Deviation}, measures how Class 0 nodes are distributed across groups. Let $\mathcal{V}_0$ be the subset of nodes belonging to Class 0. For each epoch, we compute the proportion of Class 0 nodes assigned to each group, yielding a distribution vector $\mathbf{p} = [p_0, p_1, ..., p_{K-1}]$, where $p_k = \frac{|\{i \in \mathcal{V}_0 \mid \arg\max G_{i,k} = k\}|}{|\mathcal{V}_0|}$ represents the fraction of Class 0 nodes in group $k$. We then track $\operatorname{std}(\mathbf{p})$ over training, where a larger standard deviation indicates that Class 0 nodes are more concentrated in specific groups.

\begin{figure}[t]
  \centering
  \includegraphics[width=\linewidth]{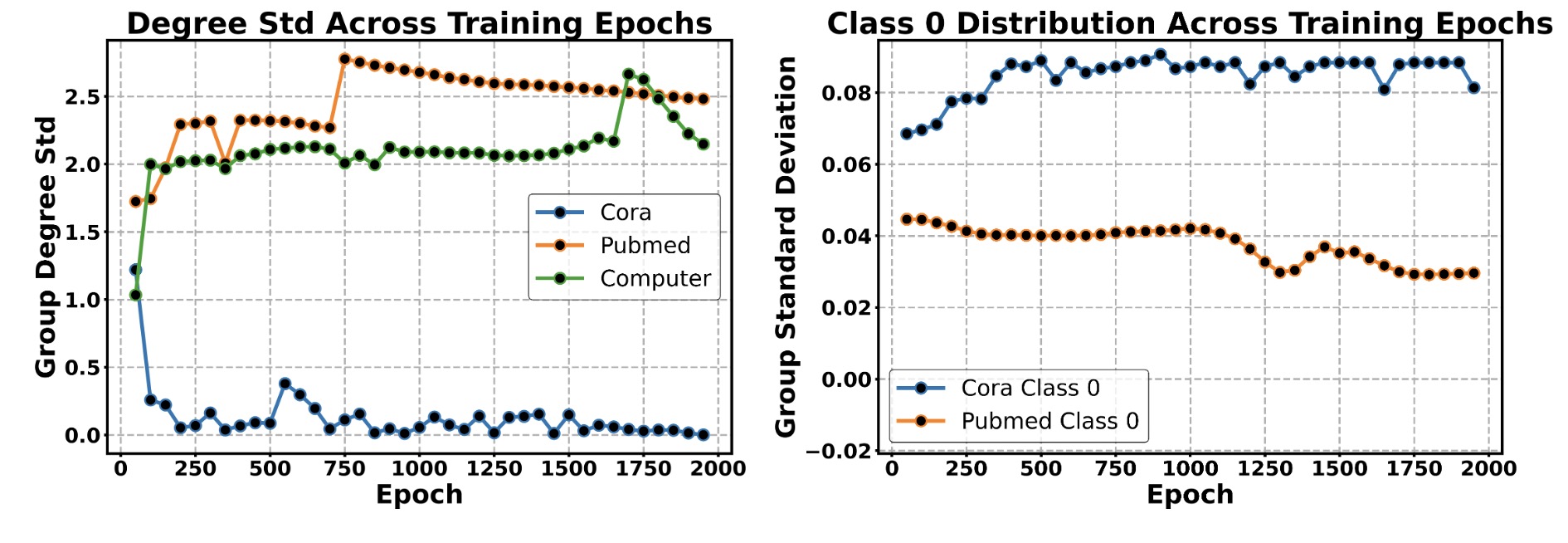}
  \vskip -1.2em
  \caption{The left plot illustrates how the degree standard deviation across groups, as learned by the group detector, evolves over training epochs. This helps visualize whether the group detector differentiates nodes based on degree information. The right plot depicts the distribution standard deviation of Class 0 nodes across different groups throughout training. This shows whether the group detector assigns Class 0 nodes consistently across groups or if their distribution shifts over time.}
  \Description{}
  \label{fig:casestudy}
  \vskip -1em
\end{figure}

The left plot in Figure~\ref{fig:casestudy} presents the degree standard deviation across groups. For Pubmed and Computer, where high-degree nodes experience calibration difficulties, we observe an increase in $\operatorname{std}(\mathbf{d})$ during training before stabilization. This suggests that the group detector actively differentiates nodes based on degree, learning that high-degree nodes require specialized calibration. In contrast, for Cora, where calibration is largely independent of degree, $\operatorname{std}(\mathbf{d})$ quickly decreases and stabilizes, indicating that the detector does not prioritize degree-based grouping. The right plot in Figure ~\ref{fig:casestudy} presents the Class 0 distribution standard deviation. In Cora, where Class 0 exhibits persistent calibration failures (as demonstrated in the preliminary analysis), we observe that $\operatorname{std}(\mathbf{p})$ increases as training progresses. This suggests that the group detector gradually isolates Class 0 nodes into distinct groups, recognizing their calibration difficulty. In contrast, for Pubmed, where Class 0 calibration improves over training, $\operatorname{std}(\mathbf{p})$ rapidly decreases, indicating that Class 0 nodes remain evenly distributed across groups, as the detector does not identify them as problematic.

These results demonstrate that the group detector learns dataset-specific miscalibration patterns, dynamically adjusting group formation based on factors that impact calibration performance.

\section{Conclusion}
In this work, we study the problem of mitigating calibration bias across different node groups to achieve more reliable confidence estimation in GNNs. To address this challenge, we propose a new framework that leverages adversarial learning to autonomously identify miscalibrated groups in graphs. Additionally, we introduce a differentiable Group-ECE loss, which explicitly guides the model to enhance the calibration within the detected miscalibrated groups. Extensive experiments demonstrate that our approach effectively improves calibration across different group levels, highlighting the importance of subgroup-aware confidence calibration for ensuring the reliability of GNN-based learning systems.

%%
%% The acknowledgments section is defined using the "acks" environment
%% (and NOT an unnumbered section). This ensures the proper
%% identification of the section in the article metadata, and the
%% consistent spelling of the heading.
% \begin{acks}
% To Robert, for the bagels and explaining CMYK and color spaces.
% \end{acks}

%%
%% The next two lines define the bibliography style to be used, and
%% the bibliography file.
%\clearpage
%\balance
% \bibliographystyle{ACM-Reference-Format}
\bibliographystyle{plain}
\bibliography{sample-base}

%%
%% If your work has an appendix, this is the place to put it.
\clearpage
\appendix

\section{Time Complexity Analysis and Training Algorithm}\label{app:timecomplexity}

\noindent\textbf{Computational Complexity Analysis.} Our model consists of two primary modules: the confidence calibration network \( f_c \), implemented as a GCN model, and the group detector \( f_g \), implemented as a fixed GIN model. Given a graph \( \mathcal{G} = (\mathcal{V}, \mathcal{E}) \), the computational complexity of both \( f_c \) and \( f_g \) is \( \mathcal{O}(|\mathcal{V}| + |\mathcal{E}|) \), as each node aggregates information from its neighbors during message passing.

During training, we optimize two loss functions: the cross-entropy loss \( \mathcal{L}_{\mathrm{CE}} \) and the proposed Group-ECE loss \( \mathcal{L}_{\mathrm{Group-ECE}} \). The cross-entropy loss is computed over the labeled node set \( \mathcal{V}_L \), leading to a complexity of \( \mathcal{O}(|\mathcal{V}_L|) \), which is typically much smaller than \( |\mathcal{V}| \) in semi-supervised settings. The Group-ECE loss requires computing confidence-accuracy misalignment for different groups, with each group iterating over all nodes. Since there are \( K \) groups, this results in a computational complexity of \( \mathcal{O}(K|\mathcal{V}|) \). Thus, the overall training complexity is given by:$\mathcal{O}(|\mathcal{V}| + |\mathcal{E}| + K|\mathcal{V}|)$. Since \( K \) is significantly smaller than \( |\mathcal{V}| \) in practice, the dominant terms remain \( \mathcal{O}(|\mathcal{V}| + |\mathcal{E}|) \).

\noindent\textbf{Training Algorithm} The training procedure follows an adversarial optimization framework, as outlined in Algorithm~\ref{alg:advcali}. At each training epoch, the calibration model \( f_c \) takes the pre-trained logits \( \mathbf{Z} \) and the graph structure \( \mathbf{A} \) as input to learn a temperature scaling parameter \( t \), producing the calibrated logits \( \mathbf{\hat{Z}} \) and probability distribution \( \mathbf{\hat{P}} \). These calibrated logits are then passed into the group detector \( f_g \), which dynamically assigns nodes to soft groups by computing the group assignment weights \( \mathbf{G} \). To enhance subgroup-level calibration, we compute the differentiable Group-ECE loss \( \mathcal{L}_{\mathrm{Group-ECE}} \) over the labeled nodes in \( \mathcal{V}_L \). The group detector \( f_g \) is trained adversarially by maximizing this loss, forcing it to identify groups where calibration discrepancies are most prominent. Simultaneously, the calibration model \( f_c \) is optimized by minimizing the final objective function, which combines the cross-entropy loss \( \mathcal{L}_{\mathrm{CE}} \) and the Group-ECE loss \( \mathcal{L}_{\mathrm{Group-ECE}} \) with a weighting factor \( \lambda \). This adversarial interplay ensures that the model not only improves global calibration but also explicitly refines confidence alignment within miscalibrated subgroups. Through this iterative process, AdvCali dynamically adjusts group assignments and refines confidence estimates, leading to improved calibration performance across diverse node groups. The final output consists of an optimized calibration model \( f_c \) and a group detector \( f_g \), which together enable effective post-hoc calibration in graph-based learning tasks.

\begin{algorithm}[t]
\caption{Training Algorithm for the Adversarial Calibration Learning (\text{AdvCali})}
\label{alg:advcali}
\begin{algorithmic}[1]
\REQUIRE Graph $\mathcal{G} = (\mathcal{V}, \mathbf{A}, \mathbf{X}, \mathcal{Y})$, labeled node set $\mathcal{V}_L$, pre-trained logits $\mathbf{Z}$, number of groups $K$, loss weight $\lambda$, max epochs $T$.
\FOR{epoch = 1 to $T$}
    \STATE Feed logits $\mathbf{Z}$ and adjacency matrix $\mathbf{A}$ to the calibration model $f_c$ to learn the temperature $t$ and obtain scaled logits $\mathbf{\hat{Z}}$ and calibrated probability $\mathbf{\hat{P}}$.
    \STATE Feed the scaled logits $\mathbf{\hat{Z}}$ into the group detector $f_g$ to compute the group assignment weights $\mathbf{G}$.
    \STATE Compute the Group-ECE loss $\mathcal{L}_{\mathrm{Group-ECE}}$ on labeled set $\mathcal{V}_{L}$ according to Eq.~(\ref{eq:groupece}).
    \STATE Update the group detector $f_g$ adversarially by maximizing $\mathcal{L}_{\mathrm{Group-ECE}}$.
    \STATE Compute the cross-entropy loss $\mathcal{L}_{\mathrm{CE}}$ on labeled set $\mathcal{V}_{L}$ and form the final loss function $\mathcal{L}$ in Eq.~(\ref{finalobjective}).
    \STATE Update the calibration model $f_c$ by minimizing $\mathcal{L}$.
\ENDFOR
\STATE \textbf{Return:} Optimized calibration model $f_c$ and group detector $f_g$.
\end{algorithmic}
\end{algorithm}

\section{Dataset Details}\label{sec:append_dataset}

\begin{table}[!ht]
\caption{Dataset Statistics.} 
\vskip -0.1in
\centering
\begin{tabular}{lccccc}
\toprule
\textbf{Dataset} & \textbf{\# Nodes} & \textbf{\# Edges} & \textbf{\# Features} & \textbf{\# Classes} \\
\midrule
Cora & 2,708 & 5,278 & 1,433 & 7 \\
CiteSeer & 3,327 & 4,552 & 3,703 & 6 \\
PubMed & 19,717 & 44,324 & 500 & 3 \\
Photo & 7,487 & 119,043 & 745 & 8 \\
Computer & 13,381 & 245,778 & 767 & 10 \\
CS & 18,333 & 81,894 & 6,805 & 15 \\
Physics & 34,493 & 247,962 & 8,415 & 5 \\
Cora-Full & 19,793 & 126,842 & 8,710 & 70 \\
\bottomrule
\end{tabular}
\label{tab:dataset}
\end{table}

In this section, we supplement the missing details of our benchmark datasets. The dataset statistics can be found in Table~\ref{tab:dataset}.
\begin{itemize}[leftmargin=*]
    \item \textbf{Cora.} Cora~\cite{yang2016revisiting} is a multi-class classification dataset that categorizes AI-related papers based on their citation relations. Each publication is represented by a 0/1-valued word vector, where each entry indicates the presence or absence of the corresponding word from the dictionary.
    
    \item \textbf{Citeseer.} Citeseer~\cite{yang2016revisiting} is a multi-class classification dataset that categorizes CS papers into subfields based on their citation relations. Each publication is represented by a 0/1-valued word vector, where each entry indicates the presence or absence of the corresponding word from the dictionary.
    
    \item \textbf{PubMed.} PubMed~\cite{yang2016revisiting} is a multi-class classification dataset derived from the PubMed database, focusing on papers related to diabetes. Each publication in the dataset is described by a TF-IDF word vector from a dictionary with 500 unique words.
    
    \item \textbf{Photo.} Amazon Photo~\cite{shchur2018pitfalls} is a multi-class classification dataset sourced from the Amazon e-commerce platform, categorizing photos into different types. Each node represents a photo, each edge represents a frequent co-purchase relation, and the features are bag-of-words features.
    
    \item \textbf{Computer.} Amazon Computer~\cite{shchur2018pitfalls} is another multi-class classification dataset sourced from the Amazon e-commerce platform, categorizing computers into different types. Each node represents a computer, each edge represents a frequent co-purchase relation, and the features are bag-of-words features.
    
    \item \textbf{CS.} Coauthor CS~\cite{shchur2018pitfalls} is a multi-class classification dataset that categorizes computer science papers into the most active fields of study for each author based on the Microsoft Academic Graph~\cite{sandulescu2016predicting}. In this dataset, nodes represent publications and edges represent citation relations. The node features correspond to paper keywords.
    
    \item \textbf{Physics.} Coauthor Physics~\cite{shchur2018pitfalls} is a multi-class classification dataset that categorizes physics papers into the most active fields of study for each author based on the Microsoft Academic Graph~\cite{sandulescu2016predicting}. The meaning of its nodes, edges, and features is similar to that of the Coauthor CS dataset.
    
    \item \textbf{Cora-Full.} Cora-Full~\cite{bojchevski2018deep} is an expanded version of the Cora dataset, which includes all papers and classes, not just those related to AI. The meaning of nodes, edges, and features is similar to that of its smaller version.
\end{itemize}

\section{Relationship between Group-ECE and ECE}\label{sec:append_proof4}
In this section, we examine the relationship between our proposed Group-ECE loss in Eq.~\ref{eq:groupece} and the original ECE objective in Definition~\ref{def:orig_ece}. Specifically, the Group-ECE loss can be expressed in a more general form:
\begin{align}\label{eq:groupece_append}
\mathcal{L}_{\mathrm{Group\text{-}ECE}} := \sum_{j=1}^K \frac{\mathbf{G}_{:, j}^\top\cdot\mathbf{1}_N}{|\mathcal{V}|} \cdot \mathrm{dist}( \mathrm{SoftAcc}(\mathbf{G}_{:, j}),~\mathrm{SoftConf}(\mathbf{G}_{:, j})),
\end{align}
where $\mathrm{dist}(\cdot,~\cdot)$ represents a generic distance measure between accuracy and confidence. As demonstrated in our ablation studies in Section~\ref{Ablation}, we find that using the squared error, i.e., $\mathrm{dist}(x_1, x_2) = (x_1 - x_2)^2$, yields superior calibration performance, leading to the specific form of our loss function in Eq.~\ref{eq:groupece} in Section~\ref{sec:gb_ece}. 

Under appropriate conditions, the Expected Calibration Error in Definition~\ref{def:orig_ece} is recovered as a special case of the Group-ECE loss in Eq.~\ref{eq:groupece_append}. This demonstrates the flexibility of our formulation, which generalizes ECE while remaining consistent with its original definition. In particular, when group assignments are hard and the distance metric is absolute error, the conventional ECE loss naturally emerges as a non-differentiable instance of our framework. This theoretical connection provides a key interpretation of Group-ECE from the perspective of the original ECE, clearly distinguishing our approach from previous objectives. The formal statement is presented below.

\begin{proposition}[Relationship between Group-ECE and ECE]\label{prop:equiv_ece_append}
    When the group weight matrix $\mathbf{G}$ has dimensions $\mathbb{R}^{N\times M}$ and takes values $\mathbf{G}_{i, j}= \mathbbm{1}(v_i \in \mathcal{B}_j)$, and the distance metric is the absolute error, i.e., $\mathrm{dist}(x_1, x_2) = |x_1 - x_2|$, the Expected Calibration Error defined in Definition~\ref{def:orig_ece} becomes a special case of the Group-ECE loss in Eq.~\ref{eq:groupece_append}.
\end{proposition}

\begin{proof}
    The ECE score in Definition~\ref{def:orig_ece} includes $M$ terms corresponding to $M$ different bins $\mathcal{B}_m,~\forall m \in \{1, \ldots, M\}$. To complete the proof, we show that for all $m \in \{1, \ldots, M\}$, each term in the ECE is equivalent to the corresponding term in the Group-ECE under the specified conditions for $\mathbf{G}$ and $\mathrm{dist}(\cdot, \cdot)$. To do so, we first prove the equivalence between the soft mean accuracy and confidence (i.e., $\mathrm{SoftAcc}(\cdot)$ and $\mathrm{SoftConf}(\cdot)$) and their non-differentiable counterparts (i.e., $\mathrm{Acc}(\cdot)$ and $\mathrm{Conf}(\cdot)$). We then demonstrate the equivalence of the overall loss functions.

    \noindent \textbf{Part 1: Accuracy Equivalence.} For the soft mean accuracy in Eq.~\ref{eq:groupece}, we have
    \begin{align*}
        \mathrm{SoftAcc}(\mathbf{G}_{:, m}) &= (\mathbf{G}_{:, m}^\top \cdot \mathbf{1}_N)^{-1} \sum_{v_i \in \mathcal{V}_L} G_{i, m} \cdot \mathbbm{1}(y_i = \hat{y}_i) \\ 
        &= \left(\sum_{i=1}^N \mathbbm{1}(v_i \in \mathcal{B}_m)\right)^{-1} \sum_{v_i \in \mathcal{V}_L} G_{i, m} \cdot \mathbbm{1}(y_i = \hat{y}_i) \\ 
        &= \frac{1}{|\mathcal{B}_m|} \sum_{v_i \in \mathcal{V}_L} G_{i, m} \cdot \mathbbm{1}(y_i = \hat{y}_i) \\ 
        &= \frac{1}{|\mathcal{B}_m|} \sum_{v_i \in \mathcal{V}_L} \mathbbm{1}(v_i \in \mathcal{B}_m) \cdot \mathbbm{1}(y_i = \hat{y}_i) \\ 
        &= \frac{1}{|\mathcal{B}_m|} \sum_{v_i \in \mathcal{B}_m} \mathbbm{1}(y_i = \hat{y}_i) \\ 
        &= \mathrm{Acc}(\mathcal{B}_m),
    \end{align*}
    where the first equality follows from Eq.~\ref{eq:soft_acc}, the second and fourth equalities follow from the property $\mathbf{G}_{i, j} = \mathbbm{1}(v_i \in \mathcal{B}_j)$, the third and fifth equalities come from basic algebra, and the final equality follows from Definition~\ref{def:orig_ece}.

    \noindent \textbf{Part 2: Confidence Equivalence.} Using similar steps as in Part 1, the equivalence between $\mathrm{SoftConf}(\mathbf{G}_{:, m})$ and $\mathrm{Conf}(\mathcal{B}_m)$ holds trivially.

    \noindent \textbf{Part 3: Loss Function Equivalence.} For the Group-ECE loss in Eq.~\ref{eq:groupece}, we have
    \begin{align*}
        \mathcal{L}_{\mathrm{Group\text{-}ECE}} &= \sum_{m=1}^M \frac{\mathbf{G}_{:, m}^\top \cdot \mathbf{1}_N}{|\mathcal{V}|} \cdot \mathrm{dist}(\mathrm{SoftAcc}(\mathbf{G}_{:, m}), \mathrm{SoftConf}(\mathbf{G}_{:, m})) \\ 
        &= \sum_{m=1}^M \frac{\mathbf{G}_{:, m}^\top \cdot \mathbf{1}_N}{|\mathcal{V}|} \cdot \mathrm{dist}(\mathrm{Acc}(\mathcal{B}_m), \mathrm{Conf}(\mathcal{B}_m)) \\ 
        &= \sum_{m=1}^M \frac{\mathbf{G}_{:, m}^\top \cdot \mathbf{1}_N}{|\mathcal{V}|} \cdot | \mathrm{Acc}(\mathcal{B}_m) - \mathrm{Conf}(\mathcal{B}_m) | \\ 
        &= \sum_{m=1}^M |\mathcal{V}|^{-1} \cdot \left(\sum_{i=1}^N \mathbbm{1}(v_i \in \mathcal{B}_m)\right) \cdot | \mathrm{Acc}(\mathcal{B}_m) - \mathrm{Conf}(\mathcal{B}_m) | \\ 
        &= \sum_{m=1}^M \frac{|\mathcal{B}_m|}{|\mathcal{V}|} \cdot | \mathrm{Acc}(\mathcal{B}_m) - \mathrm{Conf}(\mathcal{B}_m) | \\ 
        &= \mathcal{L}_{\mathrm{ECE}},
    \end{align*}
    where the first equality follows from Eq.~\ref{eq:groupece}, the second equality follows from Parts 1 and 2 of the proof, the third equality comes from using the absolute error for $\mathrm{dist}(\cdot, \cdot)$, the fourth equality uses the property $\mathbf{G}_{i, j} = \mathbbm{1}(v_i \in \mathcal{B}_j)$, the fifth equality simplifies the terms using basic algebra, and the last equality follows from Definition~\ref{def:orig_ece}. Therefore, the proof is complete.
\end{proof}
\end{document}

%% file: math_command.tex
\usepackage{amsmath,amsfonts,bm,nccmath}

\def\eqref#1{equation~\ref{#1}}
\def\1{\bm{1}}

\DeclareMathAlphabet{\mathsfit}{\encodingdefault}{\sfdefault}{m}{sl}
\SetMathAlphabet{\mathsfit}{bold}{\encodingdefault}{\sfdefault}{bx}{n}